\documentclass[accepted]{uai2023} 

\usepackage[american]{babel}

\usepackage{natbib} 
\usepackage{mathtools} 
\usepackage{booktabs} 
\usepackage{tikz} 

\usepackage{hyperref}
\hypersetup{colorlinks=true,citecolor=blue,linkcolor=blue}

\usepackage{amssymb}
\usepackage{amsmath,mathrsfs,dsfont}
\usepackage{nicefrac}
\usepackage{algorithm}
\usepackage{algorithmicx}
\usepackage{algpseudocode}

\usepackage{booktabs}

\usepackage{color}
\usepackage{enumitem}

\usepackage{array}
\usepackage{graphicx,tikz}
\usepackage[mathscr]{euscript}
\usepackage{amsthm}
\PassOptionsToPackage{square,sort,comma,numbers}{natbib}
\usepackage{multirow}

\usepackage{bm}
\usepackage{bbm}
\usepackage{color}

\usepackage{color}
\usepackage{epstopdf}
\usepackage{thmtools}
\usepackage{thm-restate}
\usepackage{subfigure}
\usepackage{bbding}

\theoremstyle{plain}
\newtheorem{theorem}{Theorem}[section]
\newtheorem{proposition}[theorem]{Proposition}
\newtheorem{lemma}[theorem]{Lemma}

\theoremstyle{definition}
\newtheorem{definition}[theorem]{Definition}

\theoremstyle{remark}

\newcommand\supp{{\rm supp}}
\newcommand{\cA}{\mathcal{A}}

\newcommand{\cS}{\mathcal{S}}

\newcommand{\cC}{\mathcal{C}}

\newcommand{\cO}{\mathcal{O}}

\newcommand{\bq}{\mathbf{q}}

\newcommand{\bu}{\mathbf{u}}
\newcommand{\bv}{\mathbf{v}}

\newcommand{\bx}{\mathbf{x}}
\newcommand{\by}{\mathbf{y}}
\newcommand{\bz}{\mathbf{z}}

\newcommand{\bA}{\mathbf{A}}

\newcommand{\bD}{\mathbf{D}}

\newcommand{\bR}{\mathbf{R}}
\newcommand{\bS}{\mathbf{S}}

\newcommand{\bU}{\mathbf{U}}
\newcommand{\bp}{{\mathbf{p}}}

\newcommand{\bX}{\mathbf{X}}

\newcommand{\bzero}{\mathbf 0}

\newcommand{\bI}{{\mathbf{I}}}
\newcommand{\bL}{{\mathbf{L}}}
\newcommand{\bW}{{\mathbf{W}}}

\newcommand{\bZ}{{\mathbf{Z}}}

\newcommand{\N}{{\rm I}\kern-0.18em{\rm N}}

\newcommand{\bbP}{\mathbb{P}}
\newcommand{\h}{{\rm I}\kern-0.18em{\rm H}}
\newcommand{\K}{{\rm I}\kern-0.18em{\rm K}}
\newcommand{\p}{{\rm I}\kern-0.18em{\rm P}}
\newcommand{\E}{{\rm I}\kern-0.18em{\rm E}}
\newcommand{\Z}{{\rm Z}\kern-0.18em{\rm Z}}

\newcommand{\1}{{\rm 1}\kern-0.25em{\rm I}}

\newcommand{\pn}{\p_{\kern-0.25em n}}
\newcommand{\pnm}{\p_{\kern-0.25em n,m}}
\newcommand{\psubm}{\p_{\kern-0.25em m}}

\def\RR{\mathbb{R}}
\def\defeq {\coloneqq}
\newcommand{\circled}[1]{\small{\raisebox{.6pt}{\textcircled{\raisebox{-.8pt}{#1}}}}}
\newcommand{\stcomp}[1]{\overline{#1}}
\def\prox{\textup {prox}}
\def \supp#1{\textup{supp}(#1)}
\def\set#1{\left\{ #1 \right\}}
\def\pth#1{\left( #1 \right)}

\def\abth#1{\left | #1 \right |}
\def \eps  {\epsilon}
\newcommand{\bfm}[1]{\ensuremath{\mathbf{#1}}}
\def\bm{\bfm m}

\newcommand{\beq}{\begin{equation}}
\newcommand{\eeq}{\end{equation}}
\newcommand{\beqa}{\begin{eqnarray}}
\newcommand{\eeqa}{\end{eqnarray}}
\newcommand{\beqas}{\begin{eqnarray*}}
\newcommand{\eeqas}{\end{eqnarray*}}
\def\bal#1\eal{\begin{align}#1\end{align}}
\def\bals#1\eals{\begin{align*}#1\end{align*}}
\def\bsal#1\esal{\begin{small}\begin{align}#1\end{align}\end{small}}
\def\bsals#1\esals{\begin{small}\begin{align*}#1\end{align*}\end{small}}
\def\bsfal#1\esfal{\begin{small}\begin{flalign}#1\end{flalign}\end{small}}

\newcommand{\BigO}[1]{{\operatorname{O}}}

\DeclareMathOperator*{\argmin}{arg\,min}

\def\norm#1#2{{\left\|#1\right\|}_{#2}}

\def\lonenorm#1{\norm{#1}{1}}
\def\ltwonorm#1{\norm{#1}{2}}
\def\fnorm#1{\norm{#1}{\textup{F}}}

\title{Locally Regularized Sparse Graph by Fast Proximal Gradient Descent}

\author[]{Dongfang Sun, Yingzhen Yang}
\affil[]{%
School of Computing and Augmented Intelligence\\
Arizona State University, Tempe, AZ 85281, USA \\
\texttt{\{dsun30,yingzhen.yang\}@asu.edu}
}

\begin{document}
\maketitle

\begin{abstract}
Sparse graphs built by sparse representation has been demonstrated to be effective in clustering high-dimensional data. Albeit the compelling empirical performance, the vanilla sparse graph ignores the geometric information of the data by performing sparse representation for each datum separately. In order to obtain a sparse graph aligned with the local geometric structure of data, we propose a novel Support Regularized Sparse Graph, abbreviated as SRSG, for data clustering. SRSG encourages local smoothness on the neighborhoods of nearby data points by a well-defined support regularization term. We propose a fast proximal gradient descent method to solve the non-convex optimization problem of SRSG with the convergence matching the Nesterov's optimal convergence rate of first-order methods on smooth and convex objective function with Lipschitz continuous gradient. Extensive experimental results on various real data sets demonstrate the superiority of SRSG over other competing clustering methods.
\end{abstract}

\section{Introduction}
Clustering methods based on pairwise similarity, such as K-means~\citep{Duda2000}, Spectral Clustering \citep{Ng01} and Affinity Propagation \citep{Frey07clusteringby}, segment the data in accordance with certain similarity measure between data points. The performance of similarity-based clustering largely depends on the similarity measure. Among various similarity-based clustering methods, graph-based methods~\citep{Schaeffer2007} are promising which often model data points and data similarity as vertices and edge weight of the graph respectively. Sparse graphs, where only a few edges of nonzero weights are associated with each vertex, are effective in clustering high-dimensional data. Existing works, such as Sparse Subspace Clustering (SSC) \citep{ElhamifarV13} and $\ell^{1}$-graph \citep{YanW09,ChengYYFH10}, build sparse graphs by sparse representation of each point. In these sparse graphs, the vertices represent the data points, an edge is between two vertices whenever one contributes to the spare representation of the other, and the weight of the edge is determined by the associated sparse codes. A theoretical explanation is provided by SSC, which shows that such sparse representation recovers the underlying subspaces from which the data are drawn under certain assumptions on the data distribution and angles between subspaces. When such assumptions hold, data belonging to different subspaces are disconnected in the sparse graph. A sparse similarity matrix is then obtained as the weighted adjacency matrix of the constructed sparse graph by $\ell^{1}$-graph or SSC, and spectral clustering is performed on the sparse similarity matrix to obtain the data clusters. In the sequel, we refer to $\ell^{1}$-graph and SSC as vanilla sparse graph. Vanilla sparse graph has been shown to be robust to noise and capable of producing superior results for high-dimensional data, compared to spectral clustering on the similarity produced by the widely used Gaussian kernel.
Albeit compelling performance for clustering, vanilla sparse graph is built by performing sparse representation for each data point independently without considering the geometric information of the data. High dimensional data always lie in low dimensional submanifold. The Manifold assumption \citep{BelkinNS06} has been employed in the sparse graph literature with an effort in learning a sparse graph aligned with the local geometric structure of the data. For example, Laplacian Regularized $\ell^{1}$-graph (LR-$\ell^{1}$-graph) is proposed in \citep{Yang2014-rl1graph,YangWYWCH14-laplacian-l1graph} to obtain locally smooth sparse codes in the sense of $\ell^{2}$-distance so as to improve the performance of vanilla sparse graph. The locally smooth sparse codes in LR-$\ell^{1}$-graph is achieved by penalizing large $\ell^{2}$-distance between the sparse codes of two nearby data points. However, the locally smooth sparse codes in the sense of $\ell^{2}$-distance does not encode the local geometric structure of the data into the construction of sparse graph. Intuitively, it is expected that nearby data points, or vertices, in a data manifold could exhibit locally smooth neighborhood according to the geometric structure of the data. Namely, nearby points are expected to have similar neighbors in the constructed sparse graph.

The support of the sparse code of a data point determines the neighbors it selects, and the nonzero elements of the sparse code contribute to the corresponding edge weights. This indicates that $\ell^{2}$-distance is not a suitable distance measure for sparse codes in our setting, and one can easily imagine that two sparse codes can have very small $\ell^{2}$-distance while their supports are quite different, meaning that they choose different neighbors. Motivated by the manifold assumption on the local sparse graph structure, we propose a novel Support Regularized Sparse Graph, abbreviated as SRSG, which encourages smooth local neighborhood in the sparse graph. The smooth local neighborhood is achieved by a well-defined support distance between sparse codes of nearby points in a data manifold.

The contributions of this paper are as follows. First, we propose Support Regularized Sparse Graph (SRSG) which is capable of learning a sparse graph with its local neighborhood structure aligned to the local geometric structure of the data manifold. Secondly, we propose an efficient and provable optimization algorithm, Fast Proximal Gradient Descent with Support Projection (FPGD-SP), to solve the non-convex optimization problem of SRSG. Albeit the nonsmoothness and nonconvexity of the optimization problem of SRSG, FPGD-SP still achieves Nesterov's optimal convergence rate of first-order methods on smooth and convex problems with Lipschitz continuous gradient.

The remaining parts of the paper are organized as follows. Vanilla sparse graph ($\ell^{1}$-graph) and LR-$\ell^{1}$-graph are introduced in the next section, and then the detailed formulation of SRSG is illustrated. We then show the clustering performance of SRSG, and conclude the paper. Throughout this paper, bold letters are used to denote matrices and vectors, and regular lower letter is used to denote scalars. Bold letter with superscript indicates the corresponding column of a matrix, e.g. $\bZ^i$ indicates the $i$-th column of the matrix $\bZ$, and the bold letter with subscript indicates the corresponding element of a matrix or vector. $\sigma_{\max}(\cdot)$ and $\sigma_{\min}(\cdot)$ indicate the maximum and minimum singular value of a matrix. $\fnorm{\cdot}$ and $\norm{\cdot}{p}$ denote the Frobenius norm and the $\ell^{p}$-norm, and ${\rm diag}(\cdot)$ indicates the diagonal elements of a matrix. $\supp{\bv}$ denotes the support of a vector $\bv$, which is the set of indices of nonzero elements of $\bv$. $[n]$ denotes all the natural numbers between $1$ and $n$ inclusively.

\section{Preliminaries: Vanilla Sparse Graph and Its Laplacian Regularization}
\label{sec::preliminary}
Vanilla sparse graph based methods \citep{YanW09,ChengYYFH10,ElhamifarV09,ElhamifarV13,WangX13,soltanolkotabi2014} apply the idea of sparse coding to build sparse graphs for data clustering. Given the data ${\bX}=[ {{\bx_1},\ldots ,{\bx_n}} ] \in {\RR^{d \times n}}$, robust version of vanilla sparse graph first solves the following optimization problem for some weighting parameter $\lambda_{\ell^{1}} > 0$ to obtain the sparse representation for each data point $\bx_i$:

\bal\label{eq:l1graph}
\mathop {\min }\limits_{{\bZ^{i} \in \RR^{n}, \bZ_i^{i} = 0 }} \ltwonorm{{\bx_i} - {\bX}\bZ^{i}}^2 + \lambda_{\ell^{1}} \lonenorm{\bZ^{i}}, \quad i \in [n],
\eal%
then construct a coefficient matrix $\bZ = [\bZ^1,\ldots,\bZ^n] \in \RR^{n \times n}$ with element $\bZ_{ij} = \bZ_i^j$, where $\bZ^{i}$ is the $i$-th column of $\bZ$. The diagonal elements of $\bZ$ are zero so as to avoid trivial solution $\bZ = \bI_n$ where $\bI_n$ is a $n \times n$ identity matrix.

A vanilla sparse graph $G = ( {{\bX},{\bW}} )$ is then built where ${\bX}$ is the set of vertices, $\bW$ is the weighted adjacency matrix of $G$. The edge weight $\bW_{ij}$ is set by the sparse codes according to
\bal\label{eq:W}
{\bW_{ij}}=({|{\bZ_{ij}}|+|{\bZ_{ji}}|})/{2}, \quad 1 \le i,j \le n.
\eal%

Finally, the data clusters are obtained by performing spectral clustering on the vanilla sparse graph $G$ with sparse similarity matrix $\bW$. In SSC and its geometric analysis \citep{ElhamifarV09,ElhamifarV13,soltanolkotabi2014}, it is proved that the sparse representation (\ref{eq:l1graph}) for each datum recovers the underlying subspaces from which the data are generated when the subspaces satisfy certain geometric properties in terms of the principle angle between different subspaces. When these required assumptions hold, data belonging to different subspaces are disconnected in the sparse graph, leading to the success of the subspace clustering. In practice, however, one can often empirically try the same formulation to obtain satisfactory results even without checking the assumptions.

High dimensional data often lie on or close to a submanifold of low intrinsic dimension, and existing clustering methods benefit from learning data representation aligned to its underlying manifold structure. While vanilla sparse graph demonstrates better performance than many traditional similarity-based clustering methods, it performs sparse representation for each datum independently without considering the geometric information and manifold structure of the entire data. On the other hand, in order to obtain the data embedding that accounts for the geometric information and manifold structure of the data, the manifold assumption \citep{BelkinNS06} is usually employed \citep{LiuCH10,HeX11,Zheng11,GaoTC13}.
\section{Support Regularized Sparse Graph}
In this section, we propose Support Regularized Sparse Graph (SRSG) which learns locally smooth neighborhoods in the sparse graph by virtue of locally consistent support of the sparse codes. Instead of imposing smoothness in the sense of $\ell^{2}$-distance on the sparse codes in the existing LR-$\ell^{1}$-graph, SRSG encourages locally smooth neighborhoods so as to capture the local manifold structure of the data in the construction of the sparse graph. A side effect of locally smooth neighborhoods is robustness to noise or outliers by encouraging nearby points on the manifold to choose similar neighbors in the sparse graph. Note that $\ell^{2}$-distance based graph regularization cannot enjoy this benefit since small $\ell^{2}$-distance between the sparse codes of nearby data points does not guarantee their consistent neighborhood in the sparse graph. The optimization problem of SRSG is
\bal\label{eq:srsg}
&\mathop {\min }\limits_{\bZ \in \RR^{n \times n}, {\rm diag}(\bZ) = \bzero} L(\bZ) =  \sum\limits_{i = 1}^n \ltwonorm{{\bx_i} - {\bX}\bZ^{i}}^2 + \gamma  \bR_{\bS}(\bZ),
\eal%
where $\bR_{\bS}(\bZ) \defeq \sum\limits_{i,j=1}^n {\bS_{ij} d(\bZ^i, \bZ^j)  } $ is the support regularization term, $\bS$ is the adjacency matrix of the KNN graph, $\gamma>0$ is the weighting parameter for support regularization term. $\bS_{ij} = 1$ if $\bx_j$ is among the $K$ nearest neighbors of $\bx_i$ in terms of the regular Euclidean distance, and $0$ otherwise. Each data point $\bx_i$ is normalized to have unit $\ell^{2}$-norm. $d(\bZ^i, \bZ^j)$ is the support distance between two sparse codes $\bZ^i$ and $\bZ^j$ which is defined as
\bal\label{eq:support-distance}
&d(\bZ^i, \bZ^j) = \sum\limits_{ 1 \le m \le n, m \neq i,j} (\1_{\bZ_m^{i} = 0,\bZ_m^{j} \neq 0} + \1_{\bZ_m^{i} \neq 0, \bZ_m^{j} = 0}).
\eal
SRSG encourages nearby data points to have similar neighborhoods by penalizing large support distance between every pair of nearby points in a data manifold. It can be seen from (\ref{eq:support-distance}) that a small support distance between $\bZ^i$ and $\bZ^j$ indicates that the indices of their nonzero elements are mostly the same. By the construction of sparse graph (\ref{eq:W}), it indicates that the two points $\bx_i$ and $\bx_j$ choose similar neighbors. Figure~\ref{fig:sparse-graph} further illustrates the effect of support regularization.
\begin{figure}[!ht]
\centering
\includegraphics[width=0.23\textwidth]{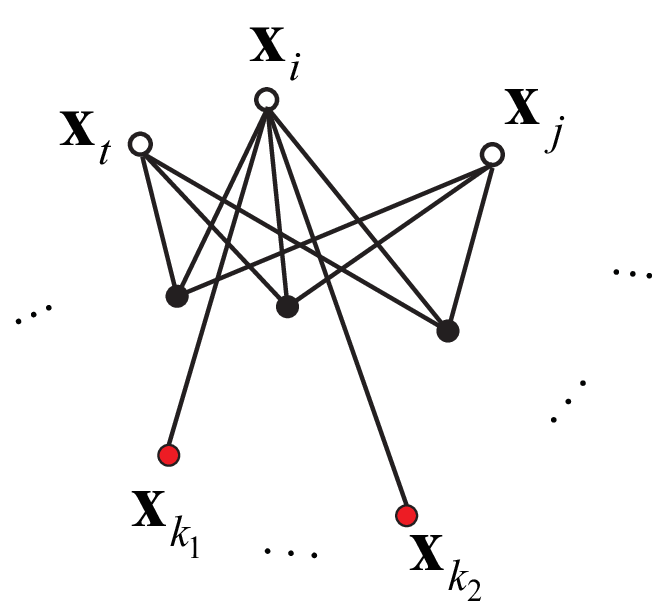}
\caption{During the construction of support regularized sparse graph, point $\bx_i$ is among the $K$ nearest neighbors of $\bx_t$ and $\bx_j$. $\bZ^t$ and $\bZ^j$ have the same support denoted by the three black dots ($\bx_{k_1}$, $\bx_{k_2}$ and $\bx_{k_3}$), suggesting the correct neighbors of $\bx_i$. By penalizing support distance between nearby points, $\bx_i$ is encouraged to choose the three black dots as neighbors in the sparse graph while discarding the wrong neighbors marked in red. }
\label{fig:sparse-graph}
\end{figure}

We use coordinate descent to optimize (\ref{eq:srsg}) with respect to $\bZ^i$, i.e. in each step the $i$-th column of $\bZ$, while fixing all the other sparse codes $\{\bZ^j\}_{j \neq i}$. In each step of coordinate descent, the optimization problem for $\bZ^i$ is
\bal\label{eq:srsg-cdi}
&\mathop {\min }\limits_{\bZ^{i} \in \RR^{n}, \bZ_i^i=0} F(\bZ^i) = \ltwonorm{{\bx_i} - {\bX}\bZ^{i}}^2 + \gamma  \bR_{\bS}(\bZ^i),
\eal%
where $\bR_{\bS}(\bZ^i) \defeq \sum\limits_{j=1}^n \bS_{ij} d(\bZ^i,\bZ^j)$. Note that $\bR_{\bS}(\bZ^i)$ can be written as
$\bR_{\bS}(\bZ^i) = \sum\limits_{t=1}^n c_{ti} \1_{\bZ_t^i \neq 0}$ up to a constant, where $c_{ti} = \sum\limits_{j=1}^n \bS_{ij} \1_{\bZ_t^j = 0} - \sum\limits_{j=1}^n  \bS_{ij} \1_{\bZ_t^j \neq 0}$.

It should be emphasized that SRSG does not use the $\ell^{1}$-norm, which is $\lonenorm{\bZ^{i}}$, to impose sparsity on $\bZ^{i}$. It is noted that the SRSG regularization term  $\bR_{\bS}(\bZ)$ induces a sparse graph where every column of $\bZ$ is sparse. This is because $\bR_{\bS}(\bZ)$ penalizes the number of different neighbors of nearby points. As a result, the remaining neighbors of every point are forced to be the common neighbors shared by its nearby points, and the number of such common neighbors is limited so that sparsity is induced. In addition, as will be illustrated in Section~\ref{sec::fast-pgd}, our proposed fast proximal method always finds a sparse solution efficiently. (\ref{eq:srsg-cdi}) is equivalent to
\bsal\label{eq:srsg-cdi-expand}
&\mathop {\min }\limits_{\bZ^{i} \in \RR^{n}, \bZ_i^i=0} F(\bZ^i) = \ltwonorm{{\bx_i} - {\bX}\bZ^{i}}^2 + \gamma \sum\limits_{t=1}^n c_{ti} \1_{\bZ_t^i \neq 0}.
\esal%
Problem (\ref{eq:srsg-cdi-expand}) is non-convex due to the non-convex regularization term $\sum\limits_{t=1}^n c_{ti} \1_{\bZ_t^i \neq 0}$, and an optimization algorithm is supposed to find a critical point for this problem. The definition of critical point for general Frechet subdifferentiable functions is defined as follows.
\begin{definition}
\label{def:subdifferential-critical-points}
\normalfont
{\rm (Subdifferential and critical points \citep{Rockafellar-Wets2009-variational-analysis})}
Given a non-convex function $f \colon \RR^n \to \RR \cup \{+\infty\}$ which is a proper and lower semi-continuous function,
\begin{itemize}[leftmargin=*]
\item For a given $\bx \in {\rm dom}f$, the Frechet subdifferential of $f$ at $\bx$, denoted by $\tilde \partial f(\bx)$, is
the set of all vectors $\bu \in \RR^n$ which satisfy
\begin{align*}
&\liminf\limits_{\by \neq \bx,\by \to \bx} \frac{f(\by)-f(\bx)-\langle \bu, \by-\bx \rangle}{\ltwonorm{\by-\bx}} \ge 0.
\end{align*}%
\item The limiting-subdifferential of $f$ at $\bx \in \RR^n$, denoted by $\partial f(\bx)$, is defined by
\begin{align*}
\partial f(\bx) &= \{\bu \in \RR^n \colon \exists \bx^k \to \bx, f(\bx^k) \to f(\bx), \\
&\tilde \bu^k \in {\tilde \partial f}(\bx^k) \to \bu\}.
\end{align*}%
\end{itemize}
The point $\bx$ is a critical point of $f$ if $\bzero \in \partial f(\bx)$.
\end{definition}

Before stating optimization algorithms that solve (\ref{eq:srsg-cdi-expand}), we introduce a simpler problem below with a simplified objective $\tilde F$. Compared to (\ref{eq:srsg-cdi-expand}), it has regularization for $\bZ_k^i$ only for positive $c_{ti}$:
\bal\label{eq:srsg-cdi-simple}
&\mathop {\min }\limits_{\bz \in \RR^n, \bz_i = 0} \tilde F(\bz) = \ltwonorm{{\bx_i} - {\bX}\bz}^2 + \gamma \sum\limits_{t\colon 1 \le t \le n, c_{ti} > 0} c_{ti} \1_{\bz_t \neq 0},
\eal%
where $\bZ^{i}$ is replaced by a simpler notation $\bz$.

Proposition 1.1 in the supplementary states that a critical point to problem (\ref{eq:srsg-cdi-simple}) has a limiting-subdifferential arbitrary close to $\bzero$. As a result, we resort to solve (\ref{eq:srsg-cdi-simple}) instead of (\ref{eq:srsg-cdi}). In the next subsection, we propose a novel proximal method with a fast convergence rate locally matching the Nesterov's optimal convergence rate for first-order methods on smooth and convex problems. In the sequel, we define $f(\bz) \triangleq \ltwonorm{{\bx_i} - {\bX}\bz}^2 $ and $h_{\gamma,c}(\bz) \triangleq \gamma \sum\limits_{t \colon 1 \le t \le n, c_{ti} > 0} c_{ti} \1_{\bz_t \neq 0}$.

\subsection{Fast Proximal Gradient Descent by Support Projection}
\label{sec::fast-pgd}
Inspired by Nesterov's accelerated Proximal Gradient Descent (PGD) method \citep{Nesterov2005-nonsmooth-optimization,Nesterov2013-gradient-composite}, we propose a novel and fast PGD method to solve (\ref{eq:srsg-cdi-simple}).
To this end, we first introduce the proximal mapping operator. The proposed fast PGD algorithm, Fast Proximal Gradient Descent with Support Projection (FPGD-SP), is described in Algorithm~\ref{alg:fpgd-sp}.
In Algorithm~\ref{alg:fpgd-sp}, the proximal mapping operator denoted by $\prox$ is defined by
\bal\label{eq:srsg-proximal-mapping}
\prox_{s h_{\gamma,c}}(\bu) &\defeq \argmin\limits_{\bv \in \RR^{n},\bv_i = 0} {\frac{1}{2s}\ltwonorm{\bv - \bu}^2 + h_{\gamma,c}(\bz)} \nonumber \\ &=T_{s,\gamma,c}(\bu),
\eal%
where $s>0$ is the step size, $T_{s,\gamma,c}$ is an element-wise hard thresholding operator. For $1 \le t\le n$,
\begin{small}
\begin{align*}
        &[T_{s,\gamma,c}(\bu)]_t=
        \left\{
        \begin{array}
                {r@{\quad:\quad}l}
                0 & {|\bu_t| \le \sqrt{2s \gamma c_{ti}} \,\, {\rm and} \,\, c_{ti} >0, \,\, {\rm or} \,\, t = i  } \\
                {\bu_t} & {\rm otherwise}
        \end{array}
        \right.
\end{align*}
\end{small}

$\bbP_{\bA}(\bu)$ in (\ref{eq:convex-tilde-vk}) of Algorithm~\ref{alg:fpgd-sp} is a novel support projection operator which returns a vector whose elements with indices in $\bA$ are the same as those in $\bu$, while all the other elements vanish. Namely, $[\bbP_{\bA}(\bu)]_k = \bu_k$ if $k \in \bA$, and $[\bbP_{\bA}(\bu)]_k = 0$ otherwise. Theorem~\ref{theorem::fpgd-sp} below shows the locally optimal convergence rate of $\cO(\frac{1}{k^2})$ achieved by FPGD-SP. We define $\cC^i \defeq \{t \colon 1 \le t \le n, c_{ti} > 0\}$
and we use $\cC$ to denote $\cC^i$ when no confusion arises.

\begin{algorithm} [!hbt]
        \caption{Fast Proximal Gradient Descent with Support Projection (FPGD-SP)
        for Problem (\ref{eq:srsg-cdi-simple})}
        \label{alg:fpgd-sp}
\begin{algorithmic}[1]

\State Input:
$\bv^{(0)} \in \RR^n$, sequence $\{\alpha_k\}_{k \ge 1}$ where $\alpha_k = \frac{2}{k+1}$ for any $k \ge 1$, step size $s > 0$, positive sequence  $\{\lambda_k\}_{k \ge 1}$ with $\lambda_k = \eta k$ for any $k \ge 1$ and another step size $\eta > 0$ such that $\lambda_k \alpha_k \le s$.
\State  Set the initial point $\bz^{(0)} = \bv^{(0)}$ and $k=1$.
\State \textbf{\bf for } $k=1,\ldots,$ \,\,\textbf{\bf do }
\bsal
&\bm^{(k)} = (1 - \alpha_k) \bz^{(k-1)} + \alpha_k \bv^{(k-1)}
\label{eq:convex-mk} \\
&\bz^{(k)} = \prox_{s h_{\gamma,c} }(\bm^{(k)} - s \nabla f(\bm^{(k)})) \label{eq:convex-zk} \\
&\tilde \bv^{(k)} = \bv^{(k-1)} - \lambda_k \nabla f(\bm^{(k)}) \label{eq:convex-tilde-vk} \\
&\bv^{(k)} = \bbP_{(\cC \cap \supp{\bz^{(k)}}) \cup \stcomp{\cC}}(\tilde \bv^{(k)}) \label{eq:convex-vk}
\esal%
\end{algorithmic}
\end{algorithm}

\begin{algorithm}[!h]
\renewcommand{\algorithmicrequire}{\textbf{Input:}}
\renewcommand\algorithmicensure {\textbf{Output:} }

\caption{Learning SRSG}
\label{alg:srsg}
\begin{algorithmic}[1]
\State Input: the data set ${\bX}=\{\bx_i\}_{i=1}^{n}$, the number of clusters $c$, the parameter $\gamma$ and $K$ for SRSG, maximum iteration number $M_c$ for coordinate descent, and maximum iteration number $M_p$ for FPGD-SP, stopping threshold $\varepsilon$.\\
\State $r=1$, initialize the sparse code matrix as ${\bZ}^{(0)} = \bZ^{(\ell^{1})}$.

\While{$r \le M_c$}
\State{Obtain ${\bZ}^{(r)}$ from ${\bZ}^{(r-1)}$ by coordinate descent. In $i$-th ($1 \le i \le n$) step of the $r$-th iteration of coordinate descent, solve (\ref{eq:srsg-cdi-simple}) by FPGD-SP descrbibed in Algorithm~\ref{alg:fpgd-sp}.}
\If{$|L(\bZ^{(r)})-L(\bZ^{(r-1)})| < \varepsilon$}
\State{\textbf{break}}
\Else
\State{$r=r+1$.}
\EndIf
\EndWhile
\State{Obtain the sub-optimal sparse code matrix ${\bZ^{*}}$ when the above iterations converge or maximum iteration number is achieved.}
\State{Build the pairwise similarity matrix by symmetrizing $\bZ^{*}$: $\bW^{*} = \frac{|\bZ^{*}|+|\bZ^{*}|^{\top}}{2}$}
\State Output: the sparse graph whose weighted adjacency matrix is $\bW^{*}$.
\end{algorithmic}
\end{algorithm}


\begin{theorem}\label{theorem::fpgd-sp}
\normalfont
Let $\{{\bz}^{(k)}\}$ be the sequence generated by Algorithm~(\ref{alg:fpgd-sp}), and suppose that there exists a constant $G$ such that $\ltwonorm{\nabla f(\bm^{(k)})} \le G$ for all $k \ge 1$. Suppose
$s < \min\set{\frac{2\tau}{G^2},\frac{1}{L_f}}$
with $L_f \defeq 2 \sigma_{\max}(\bX^{\top} \bX)$ and $\tau \defeq \gamma \min_{t \colon 1 \le t \le n, c_{ti} > 0} c_{ti}$, then there exists a finite $k_0 >0$ such that for all $k \ge k_0$, $\supp{{\bz}^{(k)}} = \cS \subseteq [n]$. Furthermore, let $\bz^*  = \argmin_{\bz \colon \supp{\bz_{\cC}} = \cS , \bz_i = 0} f(\bz)$, then
\vspace{-.1in}
\bal
&0 \le \tilde F({\bz}^{(k)}) - \tilde F(\bz^*) \le \frac{\bU^{(k_0)}}{k(k+1)}, \label{eq:fast-pgd-value-converge}
\eal%
where $\bU^{(k_0)} \defeq k_0(k_0-1)\pth{{\tilde F}(\bz^{(k_0-1)}) -{\tilde F}(\bz^*)} + \frac{\ltwonorm{\bv^{(k_0-1)} - \bz^*}^2}{\eta}$.
\end{theorem}
Theorem~\ref{theorem::fpgd-sp} shows that FPGD-SP has a fast local convergence rate of $\cO(\frac{1}{k^2})$. This convergence rate is locally optimal because $\cO(\frac{1}{k^2})$ is the Nesterov's optimal convergence rate of first-order methods on smooth and convex problems with Lipschitz continuous gradient. While the objective function $\tilde F$ is non-convex and nonsmooth, FPGD-SP still achieves the Nesterov’s optimal convergence rate.

\textbf{Key Idea in the Proof of Theorem~\ref{theorem::fpgd-sp}.} Let $\bz_{\cC} \in \RR^{\abth{\cC}}$ be the vector formed by elements of $\bz$ with indices in the set $\cC$.
The proof of Theorem~\ref{theorem::fpgd-sp} is based on the idea that when the step size for gradient descent is small enough, the support of ${\bz_{\cC}^{(k)}}$ shrinks during iterations of FPGD-SP, so the optimization through FPGD-SP can be divided into a finite number of stages wherein the support of ${\bz^{(k)}}$ belonging to the same stage remains unchanged. Therefore, the objective function $\tilde F$ is convex when restricted to the final stage. The optimal convergence rate $\cO(\frac{1}{k^2})$ in Theorem~\ref{theorem::fpgd-sp} is achieved on the final stage. 

Thanks to the property of support shrinkage, the result of FPGD-SP is always sparser than the initialization point $\bz^{(0)}$, so SRSG does not need the $\ell^{1}$-norm to impose sparsity on the solutions to (\ref{eq:srsg}) or (\ref{eq:srsg-cdi-simple}). In our experiments, the initialization point $\bz^{(0)}$ is sparse, which can be chosen as the sparse code generated by vanilla sparse graph. Due to its faster convergence rate than the vanilla PGD, we employ FPGD-SP to solve (\ref{eq:srsg-cdi-simple}) for each step of coordinate descent in the construction of SRSG.

In practice, the iteration of Algorithm~\ref{alg:fpgd-sp} is terminated when a maximum iteration number is achieved or certain stopping criterion is satisfied. When the FPGD-SP method converges or terminates, the step of coordinate descent for problem (\ref{eq:srsg-cdi-simple}) for some $\bZ^i$ is finished and the coordinate descent algorithm proceeds to optimize other sparse codes. We initialize $\bZ$ as ${\bZ}^{(0)} = \bZ_{\ell^{1}}$ and $\bZ_{\ell^{1}}$ is the sparse codes generated by vanilla sparse graph through solving (\ref{eq:l1graph}) with some proper weighting parameter $\lambda_{\ell^{1}}$. In all the experimental results shown in the next section, we empirically set $\lambda_{\ell^{1}} = 0.1$. The data clustering algorithm by SRSG is described in Algorithm~\ref{alg:srsg}. Spectral clustering is performed on the output SRSG produced by Algorithm~\ref{alg:srsg} to generate data clusters for data clustering.
\begin{figure*}[!h]
    \centering
    \includegraphics[width=0.4\textwidth]{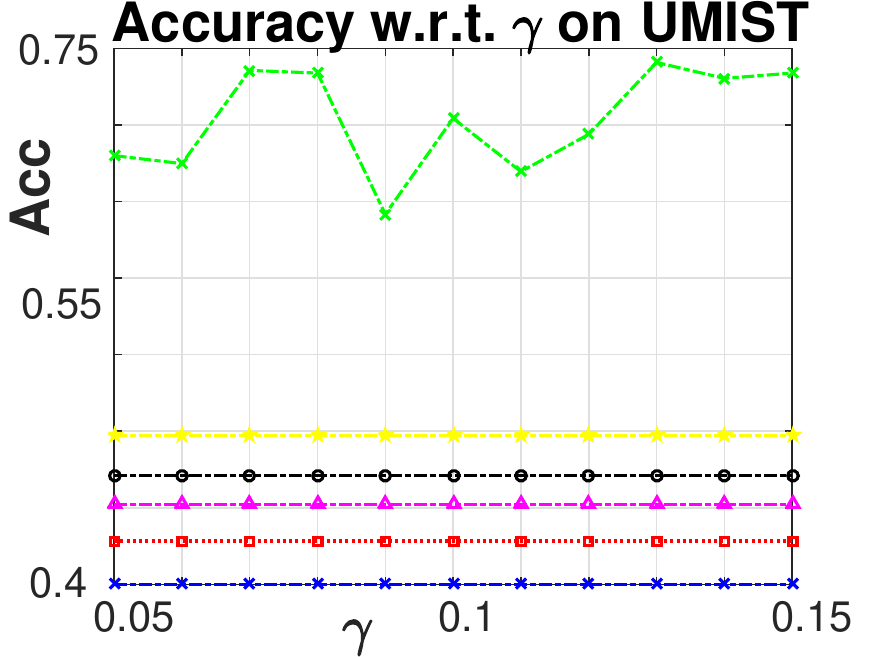}
    \includegraphics[width=0.4\textwidth]{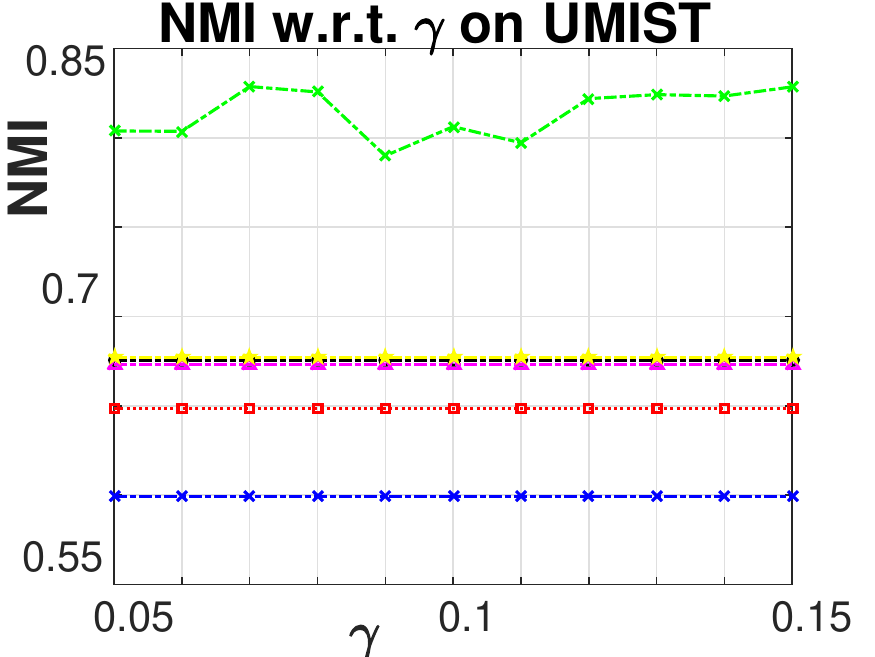}
    \includegraphics[width=0.4\textwidth]{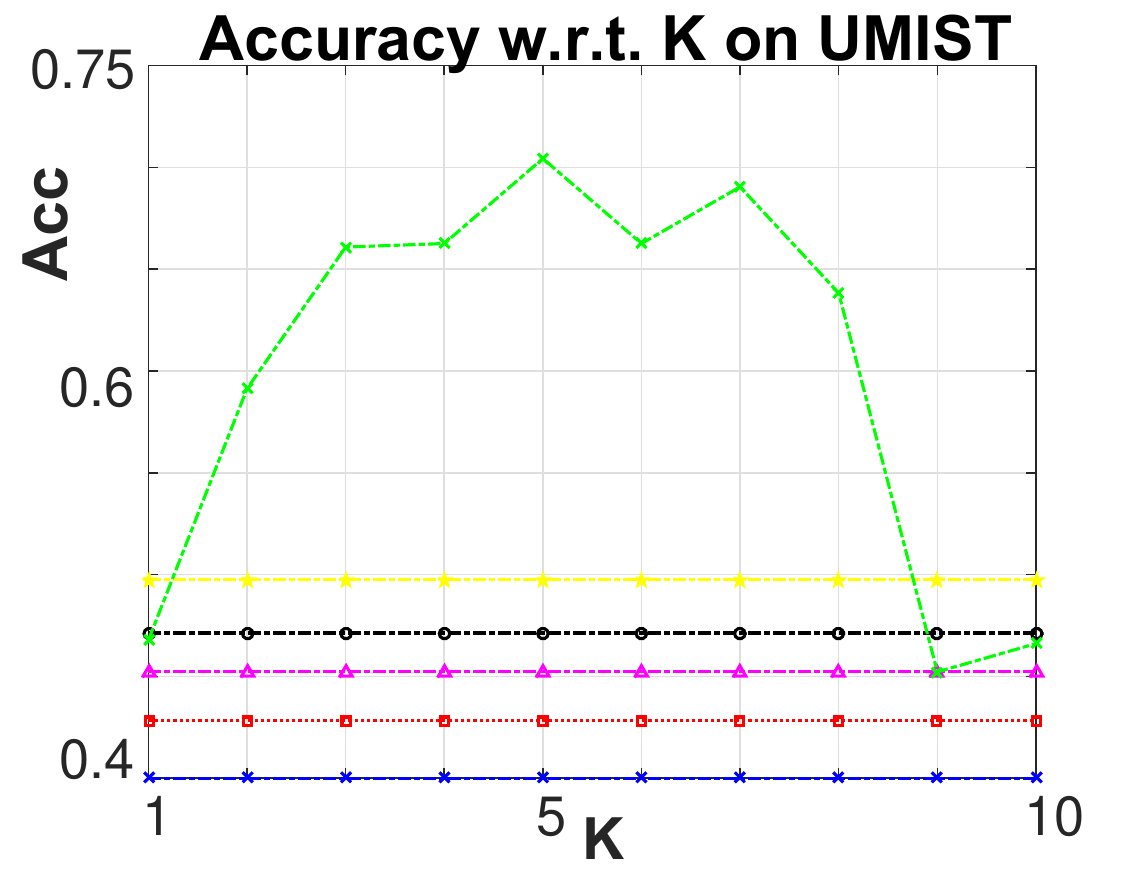}
    \includegraphics[width=0.4\textwidth]{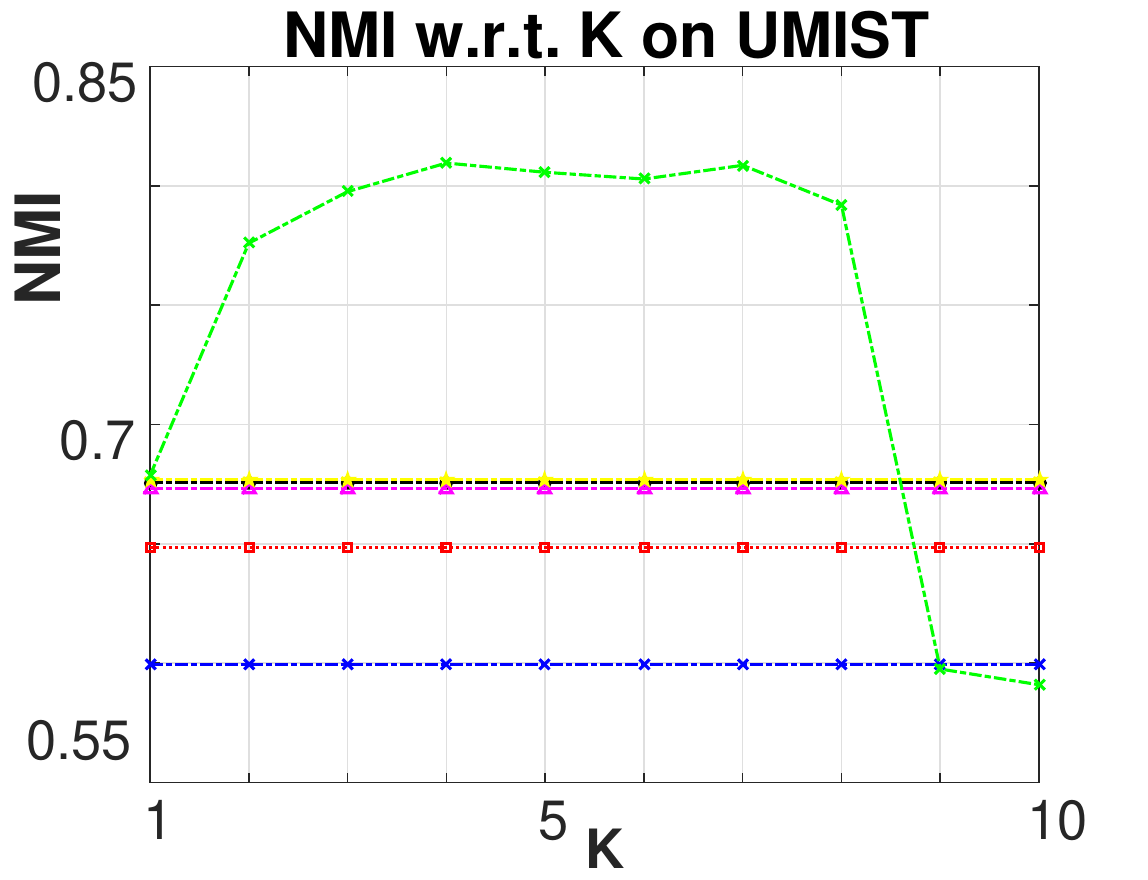}
    \includegraphics[width=0.12\textwidth]{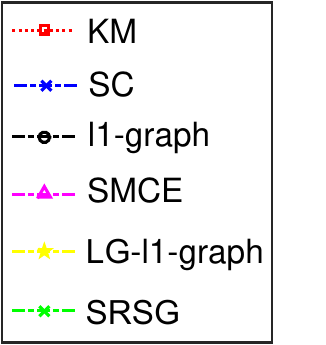}
    \caption{Parameter sensitivey on the UMIST Face Data, from left to right: Accuracy with respect to different values of $\gamma$; NMI with respect to different values of $\gamma$; Accuracy with respect to different values of $K$; NMI with respect to different values of $K$}
    \label{fig:umist-parameter-sensitivity}
\end{figure*}
\subsection{Time Complexity of Building SRSG Using FPGD-SP}
\label{sec:l0l1graph-complexity}
Let the maximum iteration number of coordinate descent be $M_c$, and maximum iteration number be $M_p$ for the FPGD-SP method used to solve (\ref{eq:srsg-cdi}). It can be verified that each iteration of Algorithm 1 has a time complexity of $\cO(d n)$ where $s_0$ is the cardinality of the support of the initialization point for FPGD-SP. The overall time complexity of running the coordinate descent for SRSG is $\cO({M_c}{M_p}n^2d)$.

\begin{table*}[!htb]
\centering
\small
\caption{ Clustering Results on COIL-20 and  COIL-100 Database. $c$ in the leftmost column indicates that the first $c$ clusters of the entire data set are used for clustering.}
\resizebox{0.7\textwidth}{!}{
\begin{tabular}{|c|c|c|c|c|c|c|c|}
  \hline
  \begin{tabular}{c}COIL-20 \\ \hline
  \# Clusters \end{tabular}

                           &Measure & KM    & SC     &$\ell^{1}$-graph   &SMCE    &LR-$\ell^{1}$-graph   &SRSG   \\\hline

  \multirow{2}{*}{c = 4}   &AC      &0.6625 &0.6701  &\textbf{1.0000}    &0.7639  &0.7188      &\textbf{1.0000} \\ \cline{2-8}
                           &NMI     &0.5100 &0.5455  &\textbf{1.0000}    &0.6741  &0.6129      &\textbf{1.0000} \\ \hline

  \multirow{2}{*}{c = 8}   &AC      &0.5157 &0.4514  &0.7986             &0.5365  &0.6858      &\textbf{0.9705} \\ \cline{2-8}
                           &NMI     &0.5342 &0.4994  &0.8950             &0.6786  &0.6927      &\textbf{0.9581} \\ \hline

  \multirow{2}{*}{c = 12}  &AC      &0.5823 &0.4954  &0.7697             &0.6806  &0.7512      &\textbf{0.8333} \\ \cline{2-8}
                           &NMI     &0.6653 &0.6096  &0.8960             &0.8066  &0.7836      &\textbf{0.9160} \\ \hline

  \multirow{2}{*}{c = 16}  &AC      &0.6689 &0.4401  &0.8264             &0.7622  &0.8142      &\textbf{0.8750} \\ \cline{2-8}
                           &NMI     &0.7552 &0.6032  &0.9294             &0.8730  &0.8511      &\textbf{0.9435} \\ \hline

  \multirow{2}{*}{c = 20}  &AC      &0.6504 &0.4271  &0.7854             &0.7549  &0.7771      &\textbf{0.8208} \\ \cline{2-8}
                           &NMI     &0.7616 &0.6202  &0.9148             &0.8754  &0.8534      &\textbf{0.9297} \\ \hline
  \begin{tabular}{c}COIL-100 \\ \hline
  \# Clusters \end{tabular}

                           &Measure & KM    & SC     &$\ell^{1}$-graph   &SMCE    &LR-$\ell^{1}$-graph   &SRSG   \\\hline

  \multirow{2}{*}{c = 20}  &AC      &0.5875 &0.4493  &0.5340             &0.6208  &0.6681      &\textbf{0.9250} \\ \cline{2-8}
                           &NMI     &0.7448 &0.6680  &0.7681             &0.7993  &0.7933      &\textbf{0.9682} \\ \hline

  \multirow{2}{*}{c = 40}  &AC      &0.5774 &0.4160  &0.5819             &0.6028  &0.5944      &\textbf{0.8465} \\ \cline{2-8}
                           &NMI     &0.7662 &0.6682  &0.7911             &0.7919  &0.7991      &\textbf{0.9484} \\ \hline

  \multirow{2}{*}{c = 60}  &AC      &0.5330 &0.3225  &0.5824             &0.5877  &0.6009      &\textbf{0.7968} \\ \cline{2-8}
                           &NMI     &0.7603 &0.6254  &0.8310             &0.7971  &0.8059      &\textbf{0.9323} \\ \hline

  \multirow{2}{*}{c = 80}  &AC      &0.5062 &0.3135  &0.5380             &0.5740  &0.5632      &\textbf{0.7970} \\ \cline{2-8}
                           &NMI     &0.7458 &0.6071  &0.8034             &0.7931  &0.7934      &\textbf{0.9240} \\ \hline

  \multirow{2}{*}{c = 100} &AC      &0.4928 &0.2833  &0.5310             &0.5625  &0.5493      &\textbf{0.7425} \\ \cline{2-8}
                           &NMI     &0.7522 &0.5913  &0.8015             &0.8057  &0.8055      &\textbf{0.9105} \\ \hline

\end{tabular}
}
\label{table:coil20-100}
\end{table*}
\section{Experimental Results}

\begin{table*}[!hbt]
\centering
\small
\caption{ Clustering Results on Various Face Datasets, where CMU Multi-PIE contains the facial images captured in four sessions (S$1$ to S$4$)}
\resizebox{0.7\textwidth}{!}{
\begin{tabular}{|c|c|c|c|c|c|c|c|}
  \hline
  \begin{tabular}{c}Yale-B \\ \hline
  \# Clusters \end{tabular}

                           &Measure & KM    & SC     &$\ell^{1}$-graph   &SMCE    &LR-$\ell^{1}$-graph   &SRSG   \\\hline

  \multirow{2}{*}{c = 10}  &AC      &0.1780 &0.1937  &0.7580             &0.3672  &0.4563      &\textbf{0.8750} \\ \cline{2-8}
                           &NMI     &0.0911 &0.1278  &0.7380             &0.3264  &0.4578      &\textbf{0.8134} \\ \hline

  \multirow{2}{*}{c = 15}  &AC      &0.1549 &0.1748  &0.7620             &0.3761  &0.4778      &\textbf{0.7754} \\ \cline{2-8}
                           &NMI     &0.1066 &0.1383  &0.7590             &0.3593  &0.5069      &\textbf{0.7814} \\ \hline

  \multirow{2}{*}{c = 20}  &AC      &0.1227 &0.1490  &0.7930             &0.3542  &0.4635      &\textbf{0.8376} \\ \cline{2-8}
                           &NMI     &0.0924 &0.1223  &0.7860             &0.3789  &0.5046      &\textbf{0.8357} \\ \hline

  \multirow{2}{*}{c = 30}  &AC      &0.1035 &0.1225  &0.8210             &0.3601  &0.5216      &\textbf{0.8475} \\ \cline{2-8}
                           &NMI     &0.1105 &0.1340  &0.8030             &0.3947  &0.5628      &\textbf{0.8652} \\ \hline

  \multirow{2}{*}{c = 38}  &AC      &0.0948 &0.1060  &0.7850             &0.3409  &0.5091      &\textbf{0.8500} \\ \cline{2-8}
                           &NMI     &0.1254 &0.1524  &0.7760             &0.3909  &0.5514      &\textbf{0.8627} \\ \hline
  \begin{tabular}{c}CMU PIE \\ \hline
  \# Clusters \end{tabular}

                           &Measure & KM    & SC     &$\ell^{1}$-graph   &SMCE    &LR-$\ell^{1}$-graph   &SRSG   \\\hline

  \multirow{2}{*}{c = 20}  &AC      &0.1327 &0.1288  &0.2329             &0.2450  &0.3076      &\textbf{0.3294} \\ \cline{2-8}
                           &NMI     &0.1220 &0.1342  &0.2807             &0.3047  &0.3996      &\textbf{0.4205} \\ \hline

  \multirow{2}{*}{c = 40}   &AC     &0.1054 &0.0867  &0.2236             &0.1931  &0.3412      &\textbf{0.3525} \\ \cline{2-8}
                           &NMI     &0.1534 &0.1422  &0.3354             &0.3038  &0.4789      &\textbf{0.4814} \\ \hline

  \multirow{2}{*}{c = 68}  &AC      &0.0829 &0.0718  &0.2262             &0.1731  &0.3012      &\textbf{0.3156} \\ \cline{2-8}
                           &NMI     &0.1865 &0.1760  &0.3571             &0.3301  &\textbf{0.5121}   &{0.4800} \\ \hline
  \begin{tabular}{c}CMU Multi-PIE \\ \hline
  \# Clusters \end{tabular}

                       &Measure & KM    & SC     &$\ell^{1}$-graph   &SMCE    &LR-$\ell^{1}$-graph   &SRSG   \\\hline


  \multirow{2}{*}{MPIE S$1$}
                           &AC      &0.1167 &0.1309  &0.5892             &0.1721  &0.4173      &\textbf{0.6815} \\ \cline{2-8}
                           &NMI     &0.5021 &0.5289  &0.7653             &0.5514  &0.7750      &\textbf{0.8854} \\ \hline

  \multirow{2}{*}{MPIE S$2$}
                           &AC      &0.1330 &0.1437  &0.6994             &0.1898  &0.5009      &\textbf{0.7364} \\ \cline{2-8}
                           &NMI     &0.4847 &0.5145  &0.8149             &0.5293  &0.7917      &\textbf{0.9048} \\ \hline

  \multirow{2}{*}{MPIE S$3$}
                           &AC      &0.1322 &0.1441  &0.6316             &0.1856  &0.4853      &\textbf{0.7138} \\ \cline{2-8}
                           &NMI     &0.4837 &0.5150  &0.7858             &0.5155  &0.7837      &\textbf{0.8963} \\ \hline

  \multirow{2}{*}{MPIE S$4$}
                           &AC      &0.1313 &0.1469  &0.6803             &0.1823  &0.5246      &\textbf{0.7649} \\ \cline{2-8}
                           &NMI     &0.4876 &0.5251  &0.8063             &0.5294  &0.8056      &\textbf{0.9220} \\ \hline

  \begin{tabular}{c}UMIST Face \\ \hline
  \# Clusters \end{tabular}

                           &Measure & KM    & SC     &$\ell^{1}$-graph   &SMCE    &LR-$\ell^{1}$-graph   &SRSG   \\\hline

  \multirow{2}{*}{c = 4}   &AC      &0.4848 &0.5691  &0.4390             &0.5203  &\textbf{0.5854}&\textbf{0.5854} \\ \cline{2-8}
                           &NMI     &0.2889 &0.4351  &0.4645             &0.3314  &\textbf{0.4686}&0.4640 \\ \hline

  \multirow{2}{*}{c = 8}   &AC      &0.4330 &0.4789  &0.4836             &0.4695  &0.5399      &\textbf{0.6948} \\ \cline{2-8}
                           &NMI     &0.5373 &0.5236  &0.5654             &0.5744  &0.5721      &\textbf{0.7333} \\ \hline

  \multirow{2}{*}{c = 12}  &AC      &0.4478 &0.4655  &0.4505             &0.4955  &0.5706      &\textbf{0.6967} \\ \cline{2-8}
                           &NMI     &0.6121 &0.6049  &0.5860             &0.6445  &0.6994      &\textbf{0.7929} \\ \hline

  \multirow{2}{*}{c = 16}  &AC      &0.4297 &0.4539  &0.4124             &0.4747  &0.4700      &\textbf{0.6544} \\ \cline{2-8}
                           &NMI     &0.6343 &0.6453  &0.6199             &0.6909  &0.6714      &\textbf{0.7668} \\ \hline

  \multirow{2}{*}{c = 20}  &AC      &0.4216 &0.4174  &0.4087             &0.4452  &0.4991      &\textbf{0.7026} \\ \cline{2-8}
                           &NMI     &0.6377 &0.6095  &0.6111             &0.6641  &0.6893      &\textbf{0.8038} \\ \hline
\end{tabular}
}
\label{table:face-results}
\end{table*}

The superior clustering performance of SRSG is demonstrated by extensive experimental results on various data sets. SRSG is compared to K-means (KM), Spectral Clustering (SC), $\ell^{1}$-graph, Sparse Manifold Clustering and Embedding (SMCE) \citep{ElhamifarV11}, and LR-$\ell^{1}$-graph introduced in Section~\ref{sec::preliminary}.

\subsection{Evaluation Metric}
Two measures are used to evaluate the performance of the clustering
methods, which are the accuracy and the Normalized Mutual Information (NMI) \citep{Zheng04}. Let the predicted label of the datum $\bx_i$ be $\hat y_i$ which is produced by the clustering method, and $y_i$ is its ground truth label. The accuracy is defined as
\bal\label{eq:accuracy}
&{Accuracy} = \frac{\1_{{\Omega}(\hat y_i) \ne y_i}}{n},
\eal%
where $\1$ is the indicator function, and $\Omega$ is the best permutation mapping
function by the Kuhn-Munkres algorithm \citep{plummer1986}. The more predicted labels match the ground truth ones, the more accuracy value is obtained.

Let $\hat X$ be the index set obtained from the predicted labels $\{\hat y_i\}_{i=1}^n$, $X$ be the index set from the ground truth labels $\{y_i\}_{i=1}^n$,
$H( {\hat X} )$ and $H( X )$ be the entropy of $\hat X$ and $X$, then the normalized mutual information (NMI) is defined as
\bal\label{eq:NMI}
&NMI( {\hat X,X} ) = \frac{{MI( {\hat X,X} )}}{{\max \{ {H( {\hat X} ),H( X )}\}}},
\eal%
where ${MI( {\hat X,X} )}$ is the mutual information between ${\hat X}$ and $X$. 

\subsection{Clustering on UCI Data Sets}
We conduct experiments on three real data sets from UCI machine learning repository \citep{Asuncion07}, i.e. Heart, Ionosphere, Breast Cancer (Breast), to reveal the clustering performance of SRSG on general data sets. The clustering results on these three data sets are shown in Table~\ref{table:uci}.
\begin{table*}[!htb]
\centering
\small
\caption{\small Clustering Results on three UCI Data Sets}
\begin{tabular}{|c|c|c|c|c|c|c|c|}
  \hline
  Data Set

                              &Measure & KM    & SC     &$\ell^{1}$-graph   &SMCE    &LR-$\ell^{1}$-graph   &SRSG   \\\hline

  \multirow{2}{*}{Heart}      &AC      &0.5889 &0.6037  &0.6370             &0.5963  &0.6259      &\textbf{0.6481} \\ \cline{2-8}
                              &NMI     &0.0182 &0.0269  &0.0529             &0.0255  &0.0475      &\textbf{0.0637} \\ \hline

  \multirow{2}{*}{Ionosphere} &AC      &0.7095 &0.7350  &0.5071             &0.6809  &0.7236      &\textbf{0.7635} \\ \cline{2-8}
                              &NMI     &0.1285 &0.2155  &0.1117             &0.0871  &0.1621      &\textbf{0.2355} \\ \hline

 \multirow{2}{*}{Breast}      &AC      &0.8541 &0.8822  &0.9033             &0.8190  &\textbf{0.9051} &\textbf{0.9051} \\ \cline{2-8}
                              &NMI     &0.4223 &0.4810  &0.5258             &0.3995  &0.5249      &\textbf{0.5333} \\ \hline

\end{tabular}
\label{table:uci}
\end{table*}
\subsection{Clustering on COIL-20 and COIL-100 Data}
COIL-20 Database has $1440$ images of resolution $32 \times 32$ for $20$ objects, and the background is removed in all images. The dimension of this data is $1024$. Its enlarged version, COIL-100 Database, contains $100$ objects with $72$ images of resolution $32 \times 32$ for each object. The images of each object were taken $5$ degrees apart when each object was rotated on a turntable. The clustering results on these two data sets are shown in Table~\ref{table:coil20-100}. It can be observed that LR-$\ell^{1}$-graph produces better clustering accuracy than $\ell^{1}$-graph, since graph regularization produces locally smooth sparse codes aligned to the local manifold structure of the data. Using the $\ell^{0}$-norm in the graph regularization term to render the sparse graph that is better aligned to the geometric structure of the data, SRSG always performs better than all other competing methods.

\subsection{Clustering on Yale-B, CMU PIE, CMU Multi-PIE, UMIST Face Data, MNIST, miniImageNet}
The Extended Yale Face Database B contains face images for $38$ subjects with $64$ frontal face images taken under different illuminations for each subject. CMU PIE face data contains cropped face images of size $32 \times 32$ for $68$ persons, and there are around $170$ facial images for each person under different illumination and expressions, with a total number of $11554$ images. CMU Multi-PIE (MPIE) data \citep{GrossMultiPIE} contains the facial images captured in four sessions. The UMIST Face Database consists of $575$ images of size $112 \times 92$ for $20$ people. Each person is shown in a range of poses from profile to frontal views - each in a separate directory labelled $1a, 1b, \ldots, 1t$ and images are numbered consecutively as they were taken. The clustering results on these four face data sets are shown in Table~\ref{table:face-results}. We conduct an extensive experiment on the popular face data sets in this subsection, and we observe that SRSG always achieve the highest accuracy, and best NMI for most cases, revealing the outstanding performance of our method and the effectiveness of manifold regularization on the local sparse graph structure. Figure 1 in the supplementary demonstrates that the sparse graph generated by SRSG effectively removes many incorrect neighbors for many data points through local smoothness of the sparse graph structure, compared to $\ell^{1}$-graph.

\subsection{Parameter Setting}
There are two essential parameters for SRSG, i.e. $\gamma$ for the $\ell^{0}$ regularization term and $K$ for building the adjacency matrix of the KNN graph. We use the sparse codes generated by $\ell^{1}$-graph with weighting parameter $\lambda_{\ell^{1}} = 0.1$ in (\ref{eq:l1graph}) to initialize both SRSG and LR-$\ell^{1}$-graph, and set $\lambda = \gamma=0.1$ in (\ref{eq:srsg}) and $K = 5$ for SRSG empirically throughout all the experiments. The maximum iteration number $M = 100$ and the stopping threshold $\varepsilon = 10^{-5}$. The weighting parameter for the $\ell^{1}$-norm in both $\ell^{1}$-graph and LR-$\ell^{1}$-graph, and the regularization weight $\gamma_{\ell^{2}}$ for LR-$\ell^{1}$-graph is chosen from $[0.1,1]$ for the best performance.

In order to investigate how the performance of SRSG varies with parameter $\gamma$ and $K$, we vary the weighting parameter $\gamma$ and $K$, and illustrate the result in Figure~{\ref{fig:umist-parameter-sensitivity}}. The performance of SRSG is noticeably better than other competing algorithms over a relatively large range of both $\lambda$ and $K$, which demonstrate the robustness of our algorithm with respect to the parameter settings. We also note that a too small $K$ (near to $1$) or too big $K$ (near to $10$) results in under regularization and over regularization.

\section*{Acknowledgement}
This material is based upon work supported by the U.S. Department of Homeland Security under Grant Award Number 17STQAC00001-07-00.
The views and conclusions contained in this document are those of the authors and should not be interpreted as necessarily representing the official policies, either expressed or implied, of the U.S. Department of Homeland Security.
This work was also supported by the 2023 Mayo Clinic
and Arizona State University Alliance for Health
Care Collaborative Research Seed Grant Program.

\section{Conclusion}
We propose a novel Support Regularized Sparse Graph (SRSG) for data clustering, which employs manifold assumption to align the sparse codes of vanilla to the local manifold structure of the data. We use coordinate descent to optimize the objective function of SRSG and propose a novel and fast Proximal Gradient Descent (PGD) with Support Projection to perform each step of the coordinate descent. Our FPGD-SP solves the non-convex optimization problem of SRSG with a proved convergence rate locally matching Nesterov's optimal convergence rate for first-order methods on smooth and convex problems. The effectiveness of SRSG for data clustering is demonstrated by extensive experiment on various real data sets.

\appendix

\section{Proofs and More Technical Results}
\label{sec:proofs}

\begin{proposition}\label{proposition:srsg-cdi-simple}
Define $\cC^{+} = \{t \colon 1 \le t \le n, c_{ti} > 0\}$, and $\cC^{-} = \{t \colon 1 \le t \le n, c_{ti} < 0\}$.
Let $\bz^*$ be a critical point of function $\tilde F$ in eq.(7) of the main paper. Then for arbitrary small positive number $\varepsilon > 0$,
$\tilde \bz^{*,\eps} \in \RR^n$ defined by
\bal\label{eq:proposition-tilde-z-eps}
\tilde \bz_k^{*,\eps} = \begin{cases}
\bz^*_k &\textup{if } \bz_k^* \neq 0 \textup{ or } k \in \cC^{+} \\
\varepsilon &\textup{otherwise}
\end{cases}
\eal
Then there exists $\bu \in \tilde \partial F(\tilde \bz_k^{*,\eps})$ for $F$ in eq.(6) of the main paper such that $\ltwonorm{\bu} \le L_f |\cC|\varepsilon$ where $L_f \defeq 2 \sigma_{\max}(\bX^{\top} \bX)$.
\end{proposition}
\begin{proof}
Since the only different elements between $\tilde \bz^{*,\eps}$ and $\bz^*$ are those with indices in $\cA = \cC^{-1} \bigcap \{k \colon \bz_k^* = 0\}$, we have
\bals
\ltwonorm{\nabla f(\tilde \bz^*) - \nabla f(\bz^*)} \le L_f \ltwonorm{\tilde \bz^* - \bz^*} \le L_f |\cC^{-1}|  \varepsilon,
\eals
where $L_f = 2 \sigma_{\max}(\bX^{\top} \bX)$. Because $\bz^*$ be a critical point of function $\tilde F$, there exists $\bq \in \tilde \partial h_{\gamma,c}$ such that $\bp \defeq \nabla f(\bz^*) + \bq = \bzero$. Define $\tilde h_{\gamma,c} = \gamma \sum\limits_{k=1}^n c_{ki} \1_{\bZ_k^i \neq 0}$. With the definition of $\tilde \bz^{*,\eps}$, we have $\tilde \bq \in \tilde \partial \tilde h_{\gamma,c}(\tilde \bz^{*,\eps})$ such that $\tilde \bq_k = 0$ for $k \in \cA$ and $\tilde \bq_k = \bq_k$ otherwise.
Moreover, $\bq_k = 0$ for all $k \in \cA$.

 Therefore,  let $\tilde \bp \triangleq \nabla f(\tilde \bz^{*,\eps}) + \tilde \bq \in \partial F(\tilde \bz^{*,\eps})$, we have
\bals
\ltwonorm{\tilde \bp} = \ltwonorm{\tilde \bp - \bp} =  \ltwonorm{\nabla f(\tilde \bz^*) - \nabla f(\bz^*)} \le L_f |\cC^{-1}|  \varepsilon.
\eals
The claim of this proposition follows with $\bu = \tilde \bp$.

\end{proof}

We repeat critical equations in the main paper and define more notations before stating the proof of Theorem 3.2.

\bals
\prox_{s h_{\gamma,c}}(\bu) \defeq \argmin\limits_{\bv \in \RR^{n},\bv_i = 0} {\frac{1}{2s}\ltwonorm{\bv - \bu}^2 + h_{\gamma,c}(\bz)} =
T_{s,\gamma,c}(\bu),
\eals%
where $s>0$ is the step size, $T_{s,\gamma,c}$ is an element-wise hard thresholding operator. For $1 \le t \le n$,
\bal\label{eq:T-thresholding}
        &[T_{s,\gamma,c}(\bu)]_t=
        \left\{
        \begin{array}
                {r@{\quad:\quad}l}
                0 & {\abth{\bu_t} \le \sqrt{2s \gamma c_{ti}} \,\, {\rm and} \,\, c_{ti} >0, \,\, {\rm or} \,\, t = i  } \\
                {\bu_t} & {\rm otherwise}
        \end{array}
        \right.
\eal%

\subsection{Proof of Theorem 3.2}

\begin{proof}[\textup{\bf Proof of Theorem 3.2}]
First of all, it can be verified that $\supp{{\bz_{\cC}}^{(k)}} \subseteq \supp{{\bz_{\cC}}^{(k-1)}}$ for all $k \ge 1$ when $s < \frac{2\tau}{G^2}$. Therefore, there exists a finite $k' \ge 1$ such that $\set{{\bz_{\cC}}^{(k)}}_{k \ge k'}$ have the same support $\cS$. We note that $\lambda$ can be also be slightly adjusted so that $\supp{\bv_{\cC}^{(k)}} = \cS$ for all $k \ge k_0$. Now we consider any $k > k'$ in the sequel, and let $\bz \in \RR^n$ be a vector such that $\supp{\bz_{\cC}} = \cS$.

Because $f$ have $L_f$-Lipschitz continuous gradient, we have
\bal\label{eq:convex-theorem-seg1}
f(\bz^{(k)}) &\le f(\bm^{(k)}) + \langle \nabla f(\bm^{(k)}), \bz^{(k)}-\bm^{(k)} \rangle + \frac{L_f}{2} \ltwonorm{\bz^{(k)}-\bm^{(k)}}^2.
\eal%

Also,
\bal\label{eq:convex-theorem-seg2}
&f(\bm^{(k)}) - (1-\alpha_k)f(\bz^{(k-1)}) - \alpha_k f(\bz) \nonumber \\
&= (1-\alpha_k) \pth{ f(\bm^{(k)}) - f(\bz^{(k-1)}) } + \alpha_k \pth{ f(\bm^{(k)}) - f(\bz) } \nonumber \\
&\stackrel{\circled{1}}{\le} (1-\alpha_k) \langle \nabla f(\bm^{(k)}), \bm^{(k)} - \bz^{(k-1)} \rangle + \alpha_k \langle \nabla f(\bm^{(k)}), \bm^{(k)} - \bz \rangle \nonumber \\
&\le \langle \nabla f(\bm^{(k)}), (1-\alpha_k) (\bm^{(k)} - \bz^{(k-1)}) + \alpha_k (\bm^{(k)} - \bz)  \rangle \nonumber \\
&= \langle \nabla f(\bm^{(k)}),  \bm^{(k)} - (1-\alpha_k) \bz^{(k-1)}  - \alpha_k \bz \rangle,
\eal
where $\circled{1}$ is due to the convexity of $f$.


We have $\tilde \bv^{(k)} = \bv^{(k-1)} - \lambda_k \nabla f(\bm^{(k)})$, and it follows that
\bal \label{eq:convex-theorem-seg3-1}
&\frac{1}{2 \lambda_k} \pth{ \ltwonorm{\bv^{(k-1)} - \bz}^2 -  \ltwonorm{\bv^{(k)} - \bz}^2 - \ltwonorm{\bv^{(k)} - \bv^{(k-1)}}^2 } \nonumber \\
&= \frac{1}{\lambda_k} \langle \bz-\bv^{(k)}, \bv^{(k)} - \bv^{(k-1)}\rangle \nonumber \\
&\stackrel{\circled{1}}{=} \frac{1}{\lambda_k}\langle \bz-\bv^{(k)}, \tilde \bv^{(k)} - \bv^{(k-1)} \rangle\nonumber \\
&= \langle \nabla f(\bm^{(k)}),  \bv^{(k)}-\bz \rangle,
\eal%
and $\circled{1}$ is due to the fact that $\supp{\bz_{\cC} - \bv_{\cC}^{(k)} } \subseteq \cS$ because $\supp{\bz_{\cC}} = \cS$, $\supp{\bv_{\cC}^{(k)}} \subseteq \cS$.

Because $\supp{\bv_{\cC}^{(k)}} \subseteq \supp{\bz_{\cC}}$, we have
\bal
h_{\gamma,c}(\bv^{(k)}) \le h_{\gamma,c}(\bz). \label{eq:convex-theorem-seg3-2}
\eal%

It follows by (\ref{eq:convex-theorem-seg3-1}) and (\ref{eq:convex-theorem-seg3-2})
that
\bal \label{eq:convex-theorem-seg3}
&\langle \nabla f(\bm^{(k)}), \bv^{(k)}-\bz \rangle + h_{\gamma,c}(\bv^{(k)}) \nonumber \\
&\le h_{\gamma,c}(\bz)
+ \frac{1}{2 \lambda_k} \pth{ \ltwonorm{\bv^{(k-1)} - \bz}^2 -  \ltwonorm{\bv^{(k)} - \bz}^2 - \ltwonorm{\bv^{(k)} - \bv^{(k-1)}}^2 }
\eal%

Similar to (\ref{eq:convex-theorem-seg3-1}),  we have
\bal \label{eq:convex-theorem-seg4-1}
\frac{1}{2 s} \pth{ \ltwonorm{\bm^{(k)} - \bz}^2 -  \ltwonorm{\bz^{(k)} - \bz}^2 - \ltwonorm{\bz^{(k)}- \bm^{(k)}}^2 }
&= \frac{1}{s} \langle \bz-\bz^{(k)}, \bz^{(k)}- \bm^{(k)}\rangle.
\eal%

For any $\bq \in \partial h_{\gamma,c}(\bz^{(k)})$, due to the fact that $\supp{\bz_{\cC}} = \supp{\bz_{\cC}^{(k)}}$,
\bal \label{eq:convex-theorem-seg4-2}
&\langle \bz-\bz^{(k)}, \bq \rangle + h_{\gamma,c}(\bz^{(k)}) = h_{\gamma,c}(\bz).
\eal%

By (\ref{eq:convex-theorem-seg4-1}) and (\ref{eq:convex-theorem-seg4-2}),
\bal \label{eq:convex-theorem-seg4-3}
&\langle \bz- \bz^{(k)}, \frac{1}{s} (\bz^{(k)}- \bm^{(k)}) + \bq \rangle + h_{\gamma,c}(\bz^{(k)}) \nonumber \\
&= h_{\gamma,c}(\bz) + \frac{1}{2 s} \pth{ \ltwonorm{\bm^{(k)} - \bz}^2 -  \ltwonorm{\bz^{(k)}- \bz}^2 - \ltwonorm{\bz^{(k)}- \bm^{(k)}}^2 }
\eal%

By the optimality condition of the proximal mapping in eq.(10) in Algorithm 1, we can choose $\bq \in \partial h_{\gamma,c}(\bz^{(k)})$ such that $\bz^{(k)}= \bm^{(k)} - s \pth{ \nabla f(\bm^{(k)}) + \bq }$. Plugging such $\bq$ in (\ref{eq:convex-theorem-seg4-3}), we have
\bal \label{eq:convex-theorem-seg4}
\langle \nabla f(\bm^{(k)}), \bz^{(k)}-\bz \rangle + h_{\gamma,c}(\bz^{(k)}) &= h_{\gamma,c}(\bz) + \frac{1}{2 s} \pth{ \ltwonorm{\bm^{(k)} - \bz}^2 -  \ltwonorm{\bz^{(k)}- \bz}^2 - \ltwonorm{\bz^{(k)}- \bm^{(k)}}^2 }
\eal%


Setting $\bz = (1-\alpha_k)\bz^{(k-1)} + \alpha_k \bv^{(k)}$
in (\ref{eq:convex-theorem-seg4}), we have
\bal \label{eq:convex-theorem-seg5}
&\langle \nabla f(\bm^{(k)}), \bz^{(k)}- (1-\alpha_k)\bz^{(k-1)} - \alpha_k \bv^{(k)} \rangle + h_{\gamma,c}(\bz^{(k)}) \nonumber \\
&\le h_{\gamma,c}((1-\alpha_k)\bz^{(k-1)} + \alpha_k \bv^{(k)}) + \frac{1}{2 s} \pth{ \ltwonorm{\bm^{(k)} - (1-\alpha_k)\bz^{(k-1)}- \alpha_k \bv^{(k)}}^2 - \ltwonorm{\bz^{(k)}- \bm^{(k)}}^2 } \nonumber \\
&\stackrel{\circled{1}}{\le} (1-\alpha_k) h_{\gamma,c} (\bz^{(k-1)}) + \alpha_k h_{\gamma,c} (\bv^{(k)}) \nonumber \\
&\phantom{=}+ \frac{1}{2 s} \pth{ \ltwonorm{\bm^{(k)} - (1-\alpha_k)\bz^{(k-1)}- \alpha_k \bv^{(k)}}^2  - \ltwonorm{\bz^{(k)}- \bm^{(k)}}^2 } \nonumber \\
&\stackrel{\circled{2}}{\le} (1-\alpha_k) h_{\gamma,c} (\bz^{(k-1)}) + \alpha_k h_{\gamma,c} (\bv^{(k)}) + \frac{1}{2 s} \pth{ \alpha_k^2 \ltwonorm{\bv^{(k)} - \bv^{(k-1)}}^2   - \ltwonorm{\bz^{(k)} - \bm^{(k)}}^2 },
\eal%
where $\circled{1}$ is due to the fact that $\supp{\bv_{\cC}^{(k)}} =
\supp{\bz_{\cC}^{(k-1)})}$ and $h_{\gamma,c}$ satisfies $h_{\gamma,c}\pth{ (1-\tau)\bu + \tau \bv } \le (1-\tau)h_{\gamma,c}(\bu) + \tau h_{\gamma,c}(\bv)$ for any two vectors $\bu$, $\bv$ with $\supp{\bu_{\cC}} = \supp{\bv_{\cC}}$ and any $\tau \in (0,1)$.
$\circled{2}$ is due to the fact that $\bm^{(k)} - (1-\alpha_k)\bz^{(k-1)}- \alpha_k \bv^{(k)} = \alpha_k (\bv^{(k-1)} - \bv^{(k)})$ according to eq.(9) in Algorithm 1. 

Computing $\alpha_k \times$ (\ref{eq:convex-theorem-seg3})
+  (\ref{eq:convex-theorem-seg5}), we have
\bal\label{eq:convex-theorem-seg6}
&\langle \nabla f(\bm^{(k)}), \bz^{(k)}- (1-\alpha_k)\bz^{(k-1)} - \alpha_k \bz \rangle + h_{\gamma,c}(\bz^{(k)}) \nonumber \\
&\le (1-\alpha_k) h_{\gamma,c} (\bz^{(k-1)}) + \alpha_k h_{\gamma,c}(\bz) \nonumber \\
&\phantom{=} + \frac{\alpha_k}{2 \lambda_k} \pth{ \ltwonorm{\bv^{(k-1)} - \bz}^2 -  \ltwonorm{\bv^{(k)} - \bz}^2 } + \pth{ \frac{\alpha_k^2}{2 s} - \frac{\alpha_k}{2 \lambda_k}} \ltwonorm{\bv^{(k)} - \bv^{(k-1)}}^2
-\frac{1}{2 s}  \ltwonorm{\bz^{(k)}- \bm^{(k)}}^2 \nonumber \\
&\stackrel{\circled{1}}{\le} (1-\alpha_k) h_{\gamma,c} (\bz^{(k-1)}) + \alpha_k h_{\gamma,c}(\bz) + \frac{\alpha_k}{2 \lambda_k} \pth{ \ltwonorm{\bv^{(k-1)} - \bz}^2 -  \ltwonorm{\bv^{(k)} - \bz}^2 } -\frac{1}{2 s}  \ltwonorm{\bz^{(k)}- \bm^{(k)}}^2,
\eal%
and $\circled{1}$ is due to $\lambda_k \alpha_k \le s$.

Combining (\ref{eq:convex-theorem-seg1}), (\ref{eq:convex-theorem-seg2}) and (\ref{eq:convex-theorem-seg6}), and noting that ${\tilde F}(\bz) = f(\bz) + h_{\gamma,c} (\bz)$, we have
\bal \label{eq:convex-theorem-seg7}
{\tilde F}(\bz^{(k)}) &\le (1-\alpha_k){\tilde F}(\bz^{(k-1)}) + \alpha_k {\tilde F}(\bz) - \pth{ \frac{1}{2 s} - \frac{L_f}{2}} \ltwonorm{\bz^{(k)}- \bm^{(k)}}^2 \nonumber \\
&\phantom{=}+ \frac{\alpha_k}{2 \lambda_k} \pth{ \ltwonorm{\bv^{(k-1)} - \bz}^2 -  \ltwonorm{\bv^{(k)} - \bz}^2 }.
\eal%

It follows by (\ref{eq:convex-theorem-seg7}) that
\bal \label{eq:convex-theorem-seg8}
{\tilde F}(\bz^{(k)}) - {\tilde F}(\bz) &\le (1-\alpha_k)\pth{ {\tilde F}(\bz^{(k-1)}) -{\tilde F}(\bz) } \nonumber \\
&\phantom{=}- \pth{ \frac{1}{2 s} - \frac{L_f}{2}} \ltwonorm{\bz^{(k)}- \bm^{(k)}}^2 + \frac{\alpha_k}{2 \lambda_k} \pth{ \ltwonorm{\bv^{(k-1)} - \bz}^2 -  \ltwonorm{\bv^{(k)} - \bz}^2 }.
\eal%

Define a sequence $\{T_k\}_{k=1}^{\infty}$ as $T_1 = 1$, and $T_k = (1 - \alpha_k) T_{k-1}$ for $k \ge 2$. Dividing both sides of (\ref{eq:convex-theorem-seg8}) by $T_k$, we have
\bal \label{eq:convex-theorem-seg9}
\frac{{\tilde F}(\bz^{(k)}) - {\tilde F}(\bz)}{T_k} &\le \frac{ {\tilde F}(\bz^{(k-1)}) -{\tilde F}(\bz) }{T_{k-1}} - \frac{1-L_f s}{2 s T_k} \ltwonorm{\bz^{(k)} - \bm^{(k)}}^2 \nonumber \\
&\phantom{=}+ \frac{\alpha_k}{2 \lambda_k T_k} \pth{ \ltwonorm{\bv^{(k-1)} - \bz}^2 -  \ltwonorm{\bv^{(k)} - \bz}^2 }.
\eal%

Since we choose $\alpha_k = \frac{2}{k+1}$, it follows that $T_k = \frac{2}{k(k+1)}$ for all $k \ge 1$. Plugging the values of $\alpha_k$ and $T_k$ in $\frac{\alpha_k}{2 \lambda_k T_k}$ in (\ref{eq:convex-theorem-seg9}), we have
\bal \label{eq:convex-theorem-seg10}
\frac{{\tilde F}(\bz^{(k)}) - {\tilde F}(\bz)}{T_k} &\le \frac{ {\tilde F}(\bz^{(k-1)}) -{\tilde F}(\bz) }{T_{k-1}} - \frac{1-L_f s}{2 s T_k} \ltwonorm{\bz^{(k)}- \bm^{(k)}}^2 \nonumber \\
&\phantom{=}+ \frac{k}{2 \lambda_k} \pth{ \ltwonorm{\bv^{(k-1)} - \bz}^2 -  \ltwonorm{\bv^{(k)} - \bz}^2 } \nonumber \\
&\stackrel{\circled{1}}{\le} \frac{ {\tilde F}(\bz^{(k-1)}) -{\tilde F}(\bz) }{T_{k-1}} - \frac{1-L_f s}{2 s T_k} \ltwonorm{\bz^{(k)}- \bm^{(k)}}^2 + \frac{k}{2 \lambda_k} \ltwonorm{\bv^{(k-1)} - \bz}^2 \nonumber \\
&\phantom{=}-  \frac{k+1}{2 \lambda_{k+1}} \ltwonorm{\bv^{(k)} - \bz}^2,
\eal%
where $\circled{1}$ is due to the condition that $\lambda_{k+1} \ge \frac{k+1}{k} \lambda_k$ for $k \ge 1$.

Set $k_0 = k'+1$. Summing the above inequality for $k=k_0,\ldots,m$ with $m \ge k_0$, we have
\bal \label{eq:convex-theorem-seg11}
\frac{{\tilde F}(\bz^{(m)}) - {\tilde F}(\bz)}{T_m} &\le \frac{ {\tilde F}(\bz^{(k_0-1)}) -{\tilde F}(\bz) }{T_{k_0-1}} + \frac{k_0\ltwonorm{\bv^{(k_0-1)} - \bz}^2}{2\lambda_{k_0}} - \sum\limits_{k=k_0}^{m} \frac{1-L_f s}{2 s T_k} \ltwonorm{\bz^{(k)} - \bm^{(k)}}^2 \nonumber \\
&\le \frac{k_0(k_0-1)\pth{{\tilde F}(\bz^{(k_0-1)}) -{\tilde F}(\bz)}}{2} + \frac{\ltwonorm{\bv^{(k_0-1)} - \bz}^2}
{2\eta} .
\eal%

Since $T_m= \frac{2}{m(m+1)}$, it follows by (\ref{eq:convex-theorem-seg11}) with $z = z^*$ that
\bal \label{eq:convex-theorem-seg12}
&{\tilde F}(\bz^{(m)}) - {\tilde F}(\bz^*) \le \frac{1}{m(m+1)} \cdot
\pth{k_0(k_0-1)\pth{{\tilde F}(\bz^{(k_0-1)}) -{\tilde F}(\bz^*)} + \frac{\ltwonorm{\bv^{(k_0-1)} - \bz^*}^2}{\eta}}.
\eal%

Changing $m$ to $k$ in (\ref{eq:convex-theorem-seg12}) completes the proof.

\end{proof}

\begin{figure*}[!hbt]
\begin{center}
\includegraphics[width=1\textwidth]{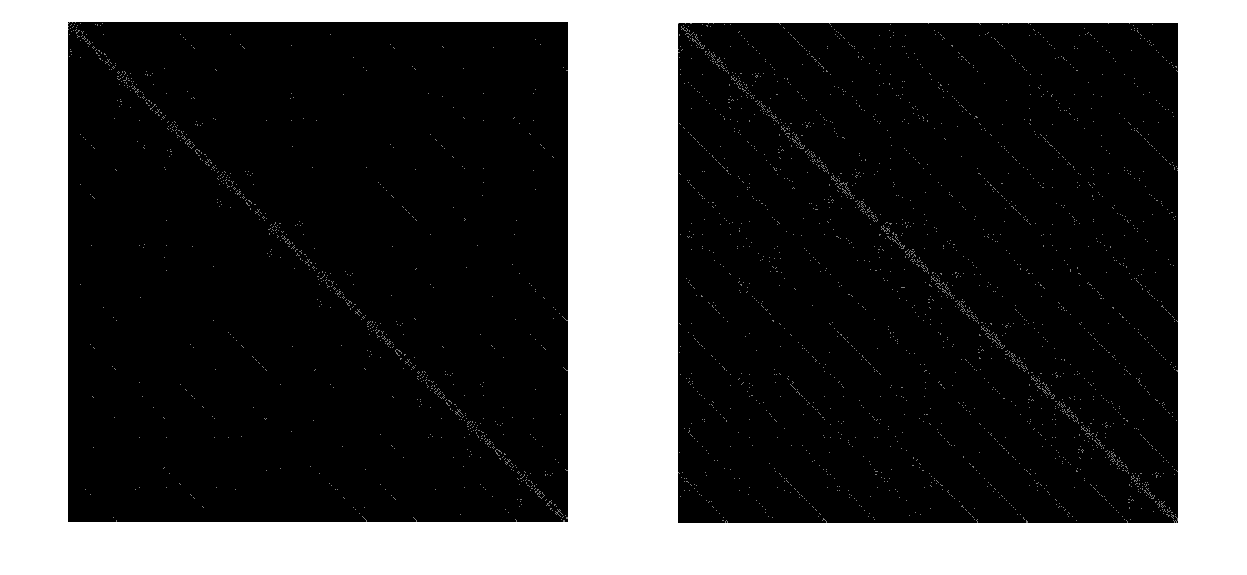}
\end{center}
   \caption{The comparison between the weighed adjacency matrix $W$ of the sparse graph produced by $\ell^{1}$-graph (right) and SRSG (left) on the Extended Yale Face Database B, where each white dot indicates an edge in the sparse graph.
   }
\label{fig:yaleb-W}
\end{figure*}
\section{Additional Illustration}

Figure~\ref{fig:yaleb-W} illustrates the comparison between the weighed adjacency matrix of $\ell^{1}$-graph and SRSG.

\bibliography{ref}

\begin{thebibliography}{25}
\providecommand{\natexlab}[1]{#1}
\providecommand{\url}[1]{\texttt{#1}}
\expandafter\ifx\csname urlstyle\endcsname\relax
  \providecommand{\doi}[1]{doi: #1}\else
  \providecommand{\doi}{doi: \begingroup \urlstyle{rm}\Url}\fi

\bibitem[A.~Asuncion(2007)]{Asuncion07}
D.J.~Newman A.~Asuncion.
\newblock {UCI} machine learning repository, 2007.

\bibitem[Belkin et~al.(2006)Belkin, Niyogi, and Sindhwani]{BelkinNS06}
Mikhail Belkin, Partha Niyogi, and Vikas Sindhwani.
\newblock Manifold regularization: A geometric framework for learning from
  labeled and unlabeled examples.
\newblock \emph{Journal of Machine Learning Research}, 7:\penalty0 2399--2434,
  2006.

\bibitem[Cheng et~al.(2010)Cheng, Yang, Yan, Fu, and Huang]{ChengYYFH10}
Bin Cheng, Jianchao Yang, Shuicheng Yan, Yun Fu, and Thomas~S. Huang.
\newblock Learning with l1-graph for image analysis.
\newblock \emph{IEEE Transactions on Image Processing}, 19\penalty0
  (4):\penalty0 858--866, 2010.

\bibitem[Duda et~al.(2000)Duda, Hart, and Stork]{Duda2000}
Richard~O. Duda, Peter~E. Hart, and David~G. Stork.
\newblock \emph{Pattern Classification (2Nd Edition)}.
\newblock Wiley-Interscience, 2000.

\bibitem[Elhamifar and Vidal(2009)]{ElhamifarV09}
Ehsan Elhamifar and Ren{\'e} Vidal.
\newblock Sparse subspace clustering.
\newblock In \emph{CVPR}, pages 2790--2797, 2009.

\bibitem[Elhamifar and Vidal(2011)]{ElhamifarV11}
Ehsan Elhamifar and Ren{\'e} Vidal.
\newblock Sparse manifold clustering and embedding.
\newblock In \emph{NIPS}, pages 55--63, 2011.

\bibitem[Elhamifar and Vidal(2013)]{ElhamifarV13}
Ehsan Elhamifar and Ren{\'{e}} Vidal.
\newblock Sparse subspace clustering: Algorithm, theory, and applications.
\newblock \emph{{IEEE} Trans. Pattern Anal. Mach. Intell.}, 35\penalty0
  (11):\penalty0 2765--2781, 2013.

\bibitem[Frey and Dueck(2007)]{Frey07clusteringby}
Brendan~J. Frey and Delbert Dueck.
\newblock Clustering by passing messages between data points.
\newblock \emph{Science}, 315:\penalty0 2007, 2007.

\bibitem[Gao et~al.(2013)Gao, Tsang, and Chia]{GaoTC13}
Shenghua Gao, Ivor Wai-Hung Tsang, and Liang-Tien Chia.
\newblock Laplacian sparse coding, hypergraph laplacian sparse coding, and
  applications.
\newblock \emph{IEEE Trans. Pattern Anal. Mach. Intell.}, 35\penalty0
  (1):\penalty0 92--104, 2013.

\bibitem[Gross et~al.(2010)Gross, Matthews, Cohn, Kanade, and
  Baker]{GrossMultiPIE}
Ralph Gross, Iain Matthews, Jeffrey Cohn, Takeo Kanade, and Simon Baker.
\newblock Multi-pie.
\newblock \emph{Image Vision Comput.}, 28\penalty0 (5):\penalty0 807--813, May
  2010.

\bibitem[He et~al.(2011)He, Cai, Shao, Bao, and Han]{HeX11}
Xiaofei He, Deng Cai, Yuanlong Shao, Hujun Bao, and Jiawei Han.
\newblock Laplacian regularized gaussian mixture model for data clustering.
\newblock \emph{Knowledge and Data Engineering, IEEE Transactions on},
  23\penalty0 (9):\penalty0 1406--1418, Sept 2011.
\newblock ISSN 1041-4347.

\bibitem[Liu et~al.(2010)Liu, Cai, and He]{LiuCH10}
Jialu Liu, Deng Cai, and Xiaofei He.
\newblock Gaussian mixture model with local consistency.
\newblock In \emph{AAAI}, 2010.

\bibitem[Nesterov(2005)]{Nesterov2005-nonsmooth-optimization}
Yu. Nesterov.
\newblock Smooth minimization of non-smooth functions.
\newblock \emph{Mathematical Programming}, 103\penalty0 (1):\penalty0 127--152,
  May 2005.

\bibitem[Nesterov(2013)]{Nesterov2013-gradient-composite}
Yu. Nesterov.
\newblock Gradient methods for minimizing composite functions.
\newblock \emph{Mathematical Programming}, 140\penalty0 (1):\penalty0 125--161,
  Aug 2013.

\bibitem[Ng et~al.(2001)Ng, Jordan, and Weiss]{Ng01}
Andrew~Y. Ng, Michael~I. Jordan, and Yair Weiss.
\newblock On spectral clustering: Analysis and an algorithm.
\newblock In \emph{NIPS}, pages 849--856, 2001.

\bibitem[Plummer and Lov{\'a}sz(1986)]{plummer1986}
D.~Plummer and L.~Lov{\'a}sz.
\newblock \emph{Matching Theory}.
\newblock North-Holland Mathematics Studies. Elsevier Science, 1986.

\bibitem[Rockafellar and Wets(2009)]{Rockafellar-Wets2009-variational-analysis}
R.~Tyrrell Rockafellar and Roger J-B Wets.
\newblock \emph{Variational Analysis}, volume 317.
\newblock Springer Science \& Business Media, 2009.

\bibitem[Schaeffer(2007)]{Schaeffer2007}
Satu~Elisa Schaeffer.
\newblock Survey: Graph clustering.
\newblock \emph{Comput. Sci. Rev.}, 1\penalty0 (1):\penalty0 27--64, August
  2007.

\bibitem[Soltanolkotabi et~al.(2014)Soltanolkotabi, Elhamifar, and
  Candès]{soltanolkotabi2014}
Mahdi Soltanolkotabi, Ehsan Elhamifar, and Emmanuel~J. Candès.
\newblock Robust subspace clustering.
\newblock \emph{Ann. Statist.}, 42\penalty0 (2):\penalty0 669--699, 04 2014.

\bibitem[Wang and Xu(2013)]{WangX13}
Yu{-}Xiang Wang and Huan Xu.
\newblock Noisy sparse subspace clustering.
\newblock In \emph{Proceedings of the 30th International Conference on Machine
  Learning, {ICML} 2013, Atlanta, GA, USA, 16-21 June 2013}, pages 89--97,
  2013.

\bibitem[Yan and Wang(2009)]{YanW09}
Shuicheng Yan and Huan Wang.
\newblock Semi-supervised learning by sparse representation.
\newblock In \emph{SDM}, pages 792--801, 2009.

\bibitem[Yang et~al.(2014{\natexlab{a}})Yang, Wang, Yang, Han, and
  Huang]{Yang2014-rl1graph}
Yingzhen Yang, Zhangyang Wang, Jianchao Yang, Jiawei Han, and Thomas Huang.
\newblock Regularized l1-graph for data clustering.
\newblock In \emph{Proceedings of the British Machine Vision Conference}. BMVA
  Press, 2014{\natexlab{a}}.

\bibitem[Yang et~al.(2014{\natexlab{b}})Yang, Wang, Yang, Wang, Chang, and
  Huang]{YangWYWCH14-laplacian-l1graph}
Yingzhen Yang, Zhangyang Wang, Jianchao Yang, Jiangping Wang, Shiyu Chang, and
  Thomas~S. Huang.
\newblock Data clustering by laplacian regularized l1-graph.
\newblock In \emph{Proceedings of the Twenty-Eighth {AAAI} Conference on
  Artificial Intelligence, July 27 -31, 2014, Qu{\'{e}}bec City, Qu{\'{e}}bec,
  Canada.}, pages 3148--3149, 2014{\natexlab{b}}.

\bibitem[Zheng et~al.(2011)Zheng, Bu, Chen, Wang, Zhang, Qiu, and Cai]{Zheng11}
Miao Zheng, Jiajun Bu, Chun Chen, Can Wang, Lijun Zhang, Guang Qiu, and Deng
  Cai.
\newblock Graph regularized sparse coding for image representation.
\newblock \emph{IEEE Transactions on Image Processing}, 20\penalty0
  (5):\penalty0 1327--1336, 2011.

\bibitem[Zheng et~al.(2004)Zheng, Cai, He, Ma, and Lin]{Zheng04}
Xin Zheng, Deng Cai, Xiaofei He, Wei-Ying Ma, and Xueyin Lin.
\newblock Locality preserving clustering for image database.
\newblock In \emph{Proceedings of the 12th Annual ACM International Conference
  on Multimedia}, MULTIMEDIA '04, pages 885--891, New York, NY, USA, 2004. ACM.

\end{thebibliography}
\end{document}